\documentclass{article}

\usepackage{vmargin}
\setpapersize{A4}

\usepackage{amssymb}

\usepackage{url}
\usepackage{xcolor}

\usepackage[colorlinks=true,linkcolor=blue,citecolor=blue]{hyperref}
\usepackage{booktabs} 
\usepackage{amsfonts} 
\usepackage{amsmath,amsthm,graphicx,epstopdf}
\usepackage{algorithm,algorithmic,natbib}

\newcommand{\g}[1]{\boldsymbol{#1}}

\newcommand{\R}[0]{\mathbb{R}} 
 
\newcommand{\I}[1]{\mathbf{1}_{#1}}

\renewcommand{\O}[0]{\mathcal{O}}

\renewcommand{\d}[0]{d} 
\newcommand{\sign}[0]{\mbox{sign}} 
\newtheorem{theorem}{Theorem}
\newtheorem{problem}{Problem}
\newtheorem{assumption}{Assumption}
\newtheorem{definition}{Definition}
\newtheorem{proposition}{Proposition}
\newtheorem{lemma}{Lemma}

\newcommand{\argmin}{\operatornamewithlimits{argmin}}

\title{\bf On the exact minimization of saturated loss functions for robust regression\\ and subspace estimation}

\author{Fabien Lauer\medskip\\\small Universit\'e de Lorraine, CNRS, LORIA, F-54000 Nancy, France} 
\date{}

\begin{document}

\maketitle
\begin{abstract}
This paper deals with robust regression and subspace estimation and more precisely with the problem of minimizing a saturated loss function. In particular, we focus on computational complexity issues and show that an exact algorithm with polynomial time-complexity with respect to the number of data can be devised for robust regression and subspace estimation. This result is obtained by adopting a classification point of view and relating the problems to the search for a linear model that can approximate the maximal number of points with a given error. Approximate variants of the algorithms based on ramdom sampling are also discussed and experiments show that it offers an accuracy gain over the traditional RANSAC for a similar algorithmic simplicity. 
\end{abstract}

\section{Introduction}
\label{sec:intro}

Robust estimation is a classical problem raised by the presence of outliers in the data. Such outliers are points that do not coincide with the underlying data distribution being learned and that must be rejected in order to estimate an accurate model. A standard approach, entering the statistical framework of redescending M-estimators \citep{Rousseeuw05,Shevlyakov08}, relies on the minimization of a saturated loss function. Indeed, this saturation ensures that outliers yielding gross errors have a very limited influence on the estimation as the gradient of the loss at these points is zero. However, saturated loss functions are inherently  nonconvex and their minimization is a highly nontrivial task. For some applications, suboptimal solutions or other heuristics such as the RANdom SAmple Consensus (RANSAC) \citep{Fischler81} can provide satisfactory models. Yet, robust estimation problems also appear for instance iteratively in a bounded-error framework for problems where the data is assumed to be generated by a collection of models with unknown assignments of the data points to the models, such as in switching linear regression \citep{Bemporad05,Bako11,Lauer16,Lauer18} or subspace clustering \citep{Vidal11,Liu13,Bako14}. In such applications, the models are often estimated one by one while considering the data assigned to other models as outliers. In this context, relying on suboptimal solutions can lead to highly unsatisfactory results with many misclassifications of data points. Robust methods based on convex relaxations \citep{Liu13,Bako14,Bako16} or iteratively hard-thresholding \citep{Bhatia15} offer some guarantees but are only optimal under particular conditions on the data. 

Instead, in this paper, we aim at unconditional optimality and discuss the computational complexity of globally minimizing a saturated loss function for the robust estimation of linear models, let it be regression ones or subspaces. 
In particular, the paper focuses on the question of the existence of an algorithm with a polynomial time-complexity with respect to the number of data,~$N$. 
To this end, we devise an algorithm by enumerating all classifications of the points into two categories: those for which saturation of the loss occurs and those for which it does not. This classification point of view is also motivated by the equivalent formulation of the problem as the maximization of the number of points approximated by a linear model with a bounded error combined with the minimization of a standard (non-saturated) loss over these points only. 
Indeed, this leads to the distinction between points with error less than and greater than a predefined threshold. 
Since there are $2^N$ binary classifications of $N$ points, such a combinatorial approach based on the enumeration of all of them yields an algorithm with exponential complexity in $\O(2^N)$. Yet, we adopt its classification viewpoint and show that the number of classifications, and thus the complexity, can be reduced to a polynomial function of~$N$. 
From this classification viewpoint, the minimization of a saturated squared loss for regression can be related to the least trimmed squares estimator \citep{Rousseeuw05}, for which exact algorithms with polynomial complexity wrt. $N$ have been proposed in \cite{Hossjer95,Li05}. However, these are restricted to problems with a single variable (one-dimensional data) and work with a fixed number of inliers rather than an error threshold.

While a polynomial complexity appears convenient, the degree of the polynomials can limit the applicability of the exact algorithms. Therefore, we also briefly discuss approximate variants of the algorithms devised to leverage the computational load by avoiding the complete enumeration of the classifications through random sampling. 

\paragraph{Notation} We write vectors in lowercase bold letters and matrices in uppercase bold letters. We define $\sign(u)$ as taking value $+1$ if and only if $u\geq 0$ and $-1$ otherwise. $\sign_0(u)$ is defined similarly except that $\sign_0(0) = 0$. The indicator function $\I{A}$ is 1 when the Boolean expression $A$ is true and 0 otherwise.  

\paragraph{Paper organization} Section~\ref{sec:pb} gives the precise formulations of the regression and subspace estimation problems we consider. Then, Section~\ref{sec:exact} shows how these can be solved in polynomial time with respect to $N$. Section~\ref{sec:random} discusses the approximate variants of the algorithms and Section~\ref{sec:exp} provides a few numerical results. Finally, Section~\ref{sec:conclusion} gives concluding remarks. 

\section{Problem formulation}
\label{sec:pb}

In general terms, in an estimation problem, one can fit a model to the data by minimizing a loss function of the error between the model output and the data.\footnote{Note that we focus on problems where the dimensionality is significantly smaller than the number of data and where regularization of linear models might not be necessary. However, given the nature of the proposed approach, introducing a convex regularizer should not raise difficulties.} For instance, standard loss functions include the $\ell_p$-losses defined for $p\geq 0$ 
and all values of the error $e\in\R$ as 
\begin{equation}\label{eq:stdloss}
	\ell_{p}(e) = 
				\begin{cases}
				\I{|e| > 0},& \mbox{if } p=0 \\
				   |e|^p ,& \mbox{if } p\in(0,+\infty) .
				\end{cases}
\end{equation}

Here, we concentrate on robust estimation in the presence of outliers and formulate the problem in terms of a {\em saturated} loss function $\ell_{p,\epsilon} : \R\rightarrow \R^+$, defined for $p\in\{0,1,2\}$ by
\begin{equation}\label{eq:loss}
	\forall \epsilon>0,\quad \ell_{p,\epsilon}(e) = 
				\begin{cases}
				\I{|e| > \epsilon},& \mbox{if } p=0 \\
				 (\min ( |e|, \epsilon) )^p ,& \mbox{if } p\in\{1,2\}.
				\end{cases}
\end{equation}
Indeed, saturating the loss function limits the influence of outliers in the overall cost function to be minimized and thus on the resulting estimate. The statistical properties of these types of loss functions have been studied in the framework of redescending M-estimators, see e.g., \cite{Rousseeuw05}.
For $p=0$, this approach is also related to bounded-error estimation. Indeed, we can equivalently view it as the maximization of the number of points for which the error is small and below the threshold $\epsilon$. For $p>0$, a similar viewpoint can be taken with the additional feature that the small errors are measured by a standard $\ell_p$-loss function and further minimized.

In this paper, we will focus the discussion on the corresponding optimization problem whose difficulty comes from the nonconvexity of the saturated losses. 

The computation of the argument $e$ as a function of the model parameters and the precise form of the optimization problem depends on the specific problem considered and will be detailed next for regression and subspace estimation.

\subsection{Robust regression via saturated loss minimization}

The aim of linear regression is to estimate a linear model $f(\g x)=\g w^T\g x$ from a data set $\{(\g x_i, y_i)\}_{i=1}^N$ of regression vectors $\g x_i\in \R^d$ and target outputs $y_i \in \R$. Here, we adopt an error-minimizing approach and more precisely focus on saturated loss functions as defined above in order to confine the influence of outliers on the global cost. 
Let us define the index sets $I=\{1,\dots,N\}$ and 
\begin{equation}\label{eq:I1}
	I_1(\g w) = \{i \in I: |y_i- \g w^T \g x_i| < \epsilon \}, 
\end{equation}
before formally stating the robust regression problem we consider.
\begin{problem}[$\ell_{p,\epsilon}$-linear regression]\label{pb:min}
Given a data set $\{(\g x_i, y_i)\}_{i=1}^N \subset \R^d \times \R$ and a threshold $\epsilon>0$, find a global solution to 
\begin{equation}
	\min_{\g w\in\R^d}\ J_p(\g w) ,
\end{equation}
where
\begin{align}\label{eq:min}
	J_p(\g w) 
	&=\sum_{i=1}^N \ell_{p,\epsilon} ( y_i - \g w^T \g x_i ) \\
	&= \begin{cases} 
		N-|I_1(\g w)|,& \mbox{if } p=0 \\
		\displaystyle{\sum_{i\in I_1(\g w)} |y_i - \g w^T \g x_i |^p + \epsilon^p (N-|I_1(\g w)|)},& \mbox{if } p\in\{1,2\}
	\end{cases} .\nonumber
\end{align}
\end{problem}

The formulation of Problem~\ref{pb:min} emphasizes the connection between saturated loss minimization and bounded-error estimation, i.e., the maximization of the number of points approximated with a bounded error that are here marked with index in $I_1(\g w)$.

This also draws a connection with the classification problem of separating between points that are approximated with a bounded error by an optimal model and those that are not.
In particular, given the solution to this classification problem, i.e., $I_1(\g w^*)$ for some global minimizer $\g w^*$ of $J_p(\g w)$, a (perhaps different\footnote{Problem~\ref{pb:min} may have multiple global solutions, especially when $p=0$.}) global solution $\hat{\g w}$ can be recovered by solving Problem~\ref{pb:min} under the constraint $I_1(\g w)=I_1(\g w^*)$. Then, for $p=0$, $J_p(\g w)$ is a mere constant and it suffices to find a $\g w$ such that $| y_i - \g w^T \g x_i |<\epsilon$ for all $i\in I_1(\g w^*)$ to satisfy the constraint. Conversely, for other values of $p$, the cost function $J_p(\g w)$ simplifies to a sum of error terms over a fixed set of points plus a constant. Hence, its minimization amounts to a standard regression problem with a non-saturated loss and we can compute $\hat{\g w}$ by solving 
\begin{equation}\label{eq:regression}
	\hat{\g w} \in \argmin_{\g w\in\R^d}\ \begin{cases}
	\displaystyle{	\max_{i\in  I_1(\g w^*)} | y_i - \g w^T \g x_i | }, &\mbox{if } p =0 \\
	\displaystyle{	 \sum_{i\in I_1(\g w^*)} | y_i - \g w^T \g x_i |^p }, &\mbox{otherwise}.
		 \end{cases}
\end{equation}
Such standard problems have polynomial complexities in $\O(d^2N)$ for $p=2$ and $\O(d^4N^4)$ for $p\in\{0,1\}$.

\subsection{Robust subspace estimation via saturated loss minimization}

A $d_s$-dimensional subspace of $\R^d$ can be thought of as the column space of a $d\times d_s$ matrix $\g B$ with orthonormal columns. In this case, the projection of a vector $\g x\in\R^d$ onto the subspace can be written as $\g B\g B^T\g x$ and the corresponding scalar approximation error as $\|(\g I - \g B\g B^T\g )\g x\|$.

Therefore, subspace estimation from a data set $\{\g x_i\}_{i=1}^N$ with a fixed subspace dimension equal to $d_s$ can be set as the search for a matrix $\g B\in\R^{d\times d_s}$ such that $\g B^T \g B = \g I$ and that the approximation error is minimized over the data set.  
In the presence of outliers, a robust estimation can be obtained from the minimization of a saturated loss function (as defined in~\eqref{eq:loss}) of this approximation error. 
  
For any $\g B\in\R^{d\times d_s}$, we define the index set
\begin{equation}\label{eq:I1subspace}
	I_1(\g B) = \{i \in I: \|(\g I - \g B\g B^T) \g x_i\| < \epsilon \}, 
\end{equation}
in order to state the problem of robust subspace estimation as follows.
\begin{problem}[$\ell_{p,\epsilon}$-subspace estimation]\label{pb:subspace}
Given a data set $\{\g x_i\}_{i=1}^N \subset \R^d$, a subspace dimension $d_s$ and a threshold $\epsilon>0$, find a global solution to 
\begin{equation}
	\min_{\g B\in\R^{d\times d_s} }\ J_p^{S}(\g B) , \quad \mbox{s.t. }  \g B^T \g B = \g I,
\end{equation}
where 
\begin{align}\label{eq:minsubspace}
J_p^{S}(\g B) &=	\sum_{i=1}^N \ell_{p,\epsilon} \left( \|(\g I - \g B\g B^T) \g x_i\| \right)  \\
	&= \begin{cases} 
		N-|I_1(\g B)|,& \mbox{if } p=0 \\
		\displaystyle{\sum_{i\in I_1(\g B)} \|(\g I - \g B\g B^T) \g x_i\|^p + \epsilon^p (N-|I_1(\g B)|)  ,}& \mbox{if } p\in\{1,2\}.\nonumber
	\end{cases}
\end{align}
\end{problem}
As for robust regression, our formulation emphasizes the classification point of view: if the classification of the point indexes into $I_1(\g B^*)$ was known for some optimal $\g B^*$, a (possibly different) global solution $\hat{\g B}$ to Problem~\ref{pb:subspace} could be obtained by solving a more simple subspace estimation subproblem of the form
\begin{align}\label{eq:subspaceest}
	\hat{\g B} \in \argmin_{\g B\in\R^{d\times d_s} }&\ \begin{cases}
	\displaystyle{ \max_{i\in I_1(\g B^*)} \|(\g I - \g B\g B^T) \g x_i\| }, & \mbox{if } p= 0 \\
	\displaystyle{ \sum_{i\in I_1(\g B^*)} \|(\g I - \g B\g B^T) \g x_i\|^p }, & \mbox{otherwise}
	\end{cases}\\
	  s.t.\ & \g B^T \g B = \g I .\nonumber
\end{align}
For instance, for $p = 2$, the solution to~\eqref{eq:subspaceest} is computable in $\O(d^2N)$ time via the singular value decomposition of the matrix $\g X_1$ made of the data points $\g x_i$ with index $i\in I_1(\g B^*)$ as columns, $\g X_1=\g U\g \Sigma\g V^T$, by extracting a subset of columns $\g u_k$ from $\g U$: $\hat{\g B} = [\g u_{1}, \dots, \g u_{d_s}]$.  

\section{Exact algorithms with polynomial time-complexity with respect to $N$}
\label{sec:exact}

We now turn to the analysis of the computational complexity of Problems~\ref{pb:min}--\ref{pb:subspace}  wrt. $N$, i.e., for a fixed data dimension $d$. 
In particular, we will show that, under simple assumptions on the genericity of the point distributions, these complexities are no more than polynomial. 
\begin{assumption}\label{ass:generalposition}
In Problem~\ref{pb:min}, the points $\{[y_i,\g x_i^T]^T\}_{i=1}^N \cup\{\g 0\}$ are in general position, i.e., no hyperplane passing through the origin of $\R^{d+1}$ contains more than $d$ points from $\{[y_i,\g x_i^T]^T\}_{i=1}^N$. 
\end{assumption}

\begin{assumption}\label{ass:generalpossubspace}
In Problem~\ref{pb:subspace}, the points $\{\g x_i\}_{i=1}^N$ are in general position, i.e., no hyperplane of $\R^d$ contains more than $d$ of these points. 
\end{assumption}

Sections~\ref{sec:polyreg} and \ref{sec:polysub} will prove the existence of exact algorithms that run in polynomial time for Problems~\ref{pb:min} and~\ref{pb:subspace}. In both cases, these algorithms will be devised with the following approach. As discussed above, thanks to the formulations \eqref{eq:min} and \eqref{eq:minsubspace} expressing the objective functions in terms of the index sets $I_1(\g w)$ and $I_1(\g B)$ of points with  error less than $\epsilon$, we can compute an optimal solution $\hat{\g w}$ or $\hat{\g B}$ from the knowledge of the optimal set $I_1^*=I_1(\g w^*)$ or $I_1^*=I_1(\g B^*)$. Thus, the algorithm can perform a combinatorial search for the index set $I_1^*$ rather than a continuous optimization over $\g w$ or $\g B$. The number of sets $I_1\subseteq I$ being finite, we can enumerate them and compute the optimal continuous variables and objective function values for each one of them with the guarantee of finding a global solution. The main difficulty with this approach is that the number of sets $I_1\subseteq I$ is $2^N$ and thus exponential in $N$.
However, we will show below that the sets $I_1(\g w)$ and $I_1(\g B)$ can be obtained via linear classification for any $\g w$ and $\g B$ and that all the corresponding linear classifications can be enumerated in polynomial time wrt. $N$.

This reduction to a polynomial complexity is based on recent results from \cite{Lauer15}, where it is shown that 
the number of hyperplanes producing different classifications of $N$ points is on the order of $\O(N^d)$ in $\R^d$ and that the complexity of constructing these hyperplanes is of a similar order, i.e., all hyperplanes can be computed in $\O(N^d)$ operations.\footnote{Note that a polynomial bound on the number of linear classifications of $N$ points such as the one obtained with the celebrated Sauer's lemma and Vapnik-Chervonenkis dimension of hyperplanes is not sufficient for our purposes, since it does not provide an algorithm to explicitly enumerate the classifications.} More precisely, we will need an adaptation of these results, stated as Proposition~\ref{prop:hyperplane} below, in order to deal with linear instead of affine classifiers and to work under less restrictive conditions on the distribution of the points. 
Indeed, Proposition~3 in \cite{Lauer15} requires the points to be in general position to avoid having too many (i.e., a number  proportional to $N$)  undetermined classifications of points lying exactly on the separating hyperplane. In our framework below, we apply this proposition not to data points but to their projection in a space where this assumption cannot hold. Hence, we will instead prove that the number of projected points falling on the hyperplane cannot exceed a certain constant if the original data points are in general position and build the algorithm to explicitly deal with these points. 

\begin{proposition}\label{prop:hyperplane}
For any binary linear classifier $h(\g x) = \sign(\g h^T\g x)$, 
$\g h\in\R^{d}$, 
and any finite set of $N\geq d$ points $S = \{\g x_i\}_{i=1}^N$, there is a subset of points $S_h\subset S$ of cardinality $|S_h|= d-1$ and a separating hyperplane of normal $\g h_{S_h}$ passing through the points in $S_h$, i.e., 
\begin{equation}\label{eq:computehb}
	\forall \g x\in S_h,\quad \g h_{S_h}^T \g x = 0, \quad \mbox{with }\|\g h_{S_h}\| = 1, 
\end{equation}
which yields the same classification of $S$ in the sense that
\begin{align}\label{eq:equivclassif}
	\forall \g x_i \in S\setminus \{\g x\in S : \g h_{S_h}^T \g x = 0\},\quad h(\g x_i) = \sign( \g h_{S_h}^T \g x_i ). 
\end{align}
\end{proposition}
Note that the set $\{\g x\in S : \g h_{S_h}^T \g x = 0\}$ includes $S_h$ but can also include more than $d-1$ points if $S$ is not in general position. 
\begin{proof}[Proof sketch]
The general sketch of the proof is similar to the one of Proposition~3 in \cite{Lauer15}: the hyperplane of normal $\g h$ is transformed to the one of normal $\g h_{S_h}$ via a series of translation and rotations, while making sure that the classification remains unchanged except for the points that end up lying on the hyperplane. The first difference is that the set of classifiers is restricted to linear instead of affine ones by only considering hyperplanes passing through the origin, and thus with one less degree of freedom and one less point to choose in $S_h$ to fix the hyperplane. Technically, this removes the need for the first translation and only rotations are used to transform the hyperplane. Another difference is that, due to the lack of assumption on the distribution of the points, the number of additional points through which the hyperplane passes after a rotation is unbounded (otherwise than by $N$). Thus, the statement in~\eqref{eq:equivclassif} explicitly excludes all points on the hyperplane and appears slightly weaker than the one in \cite{Lauer15} in that it might apply to less points. 
\end{proof}

\subsection{$\ell_{p,\epsilon}$-linear regression}
\label{sec:polyreg}

We first focus on Problem~\ref{pb:min} and start with a classification-based reformulation of the set of indexes $I_1(\g w^*)$ of points with error smaller than~$\epsilon$. This relies on the construction of a classification data set 
\begin{equation}\label{eq:Z}
	\mathcal{Z}=\{\g z_i\}_{i=1}^{2N},\ \mbox{with }\g z_i = \begin{cases} [y_i-\epsilon,\ -\g x_i^T]^T, & \mbox{if } i\leq N\\ [-y_{i-N} - \epsilon, \ \g x_{i-N}^T]^T, & \mbox{if } i>N\end{cases}.
\end{equation}

\begin{lemma}\label{lem:equivclassif}
Given a parameter vector $\g w\in \R^d$, the set $I_1(\g w)$ defined in~\eqref{eq:I1} is given by
\begin{equation}\label{eq:majorityvote}
	I_1(\g w) = \left\{i\in I : q_i = q_{i+N} = -1,\quad q_i =\sign \left( \g h(\g w)^T \g z_i\right) \right\},
\end{equation}
where $\g h(\g w) = [1, \g w^T]^T$ and $\mathcal{Z}=\{\g z_i\}_{i=1}^{2N}$ as in~\eqref{eq:Z}. 
\end{lemma}
\begin{proof}
For any $i\in I$, we have $i\leq N$ and $q_i=\sign ( \g h(\g w)^T \g z_i) = \sign ( y_i-\epsilon - \g w^T\g x_i)$, while $q_{i+N} = \sign ( \g h(\g w)^T \g z_{i+N}) = \sign ( -y_i-\epsilon + \g w^T\g x_i)$. Thus, 
\begin{align*}
 |y_i- \g w^T \g x_i| < \epsilon &\Leftrightarrow  y_i - \epsilon - \g w^T\g x_i < 0 \wedge -y_i - \epsilon + \g w^T\g x_i < 0\\
 &\Leftrightarrow q_i = -1 \wedge q_{i+N} = -1
\end{align*}
and recalling the definition of $I_1(\g w)$ in~\eqref{eq:I1} completes the proof.
\end{proof}

We will also need the following. 
\begin{lemma}\label{lem:n}
Given a data set $\{(\g x_i, y_i)\}_{i=1}^N$ satisfying Assumption~\ref{ass:generalposition}, no hyperplane of $\R^{d+1}$ passing through the origin can pass through more than $d$ points of each of the sets $\mathcal{Z}_1 = \{\g z_i=[y_i-\epsilon,\ -\g x_i^T]^T\}_{i=1}^N$, $\mathcal{Z}_2=\{\g z_i= [-y_{i-N} - \epsilon, \ \g x_{i-N}^T]^T\}_{i=N+1}^{2N}$ and thus through more than $2d$ points of $\mathcal{Z} = \mathcal{Z}_1\cup\mathcal{Z}_2$.
\end{lemma}
\begin{proof}
By Assumption~\ref{ass:generalposition}, each of the sets $\mathcal{Z}_1\cup\{\g 0\}$ and $\mathcal{Z}_2\cup\{\g 0\}$ is also in general position. Thus, no hyperplane of $\R^{d+1}$ can pass through more than $d$ points in each of these sets, and hence through more than a total of $n\leq 2d$ points of $\mathcal{Z}$.
\end{proof}

Using the classification viewpoint of Lemma~\ref{lem:equivclassif}, we can state the following result which considers the case where more than $d$ data points can be approximated with error less than $\epsilon$ by a linear model. Note that there are always at least $d$ such points and that the case where there are precisely $d$ is trivial since any group of $d$ points yields an optimal solution. 
\begin{proposition}\label{prop:gloabloptfromclassif}
Assume that the global minimum of Problem~\ref{pb:min} is $J_p^* < \epsilon^p( N-d)$. Then, under Assumption~\ref{ass:generalposition}, in $\R^{d+1}$, there is a hyperplane of normal $\g h\in \R^{d+1}$ passing through the origin and 
$n\in[d,2d]$ points of $\mathcal{Z}$ as in~\eqref{eq:Z} such that $h_1>0$ and 
\begin{itemize}
	\item[i)] a global solution can be computed by solving $2^n\leq 2^{2d}$ standard subproblems~\eqref{eq:regression} that can be built from $\g h$, 
	\item[ii)] for $p=0$, $\g w = \g h_{2:d+1} / h_1$ is an approximate solution with $J_0(\g w) \leq J_0^* + n\leq J_0^* + 2d$.
\end{itemize}
\end{proposition}

\begin{proof}
First note that, as a direct consequence of Lemma~\ref{lem:n}, we always have $n\leq 2d$.

{\em Part i)}  Let $\mathcal{W}_p^*$ be the set of global minimizers of Problem~\ref{pb:min}. By Lemma~\ref{lem:equivclassif}, for all $\g w^*\in\mathcal{W}_p^*$ there is a hyperplane of normal $\g h(\g w^*)$ classifying the $2N$ points of $\mathcal{Z}$ into $I_1(\g w^*)$ and $I\setminus I_1(\g w^*)$. By Proposition~\ref{prop:hyperplane}, there is an equivalent hyperplane of normal $\g h^*$ passing through $n\geq d$ points of $\R^{d+1}$, such that 
\begin{equation}\label{eq:signeq}
	\forall i\in \{1,\dots, 2N\}\setminus I_0,\quad \sign\left(\g z_i^T\g h(\g w^*)\right) = \sign\left(\g z_i^T\g h^*\right), 
\end{equation}
where $I_0=\{i : \g z_i^T\g h^* = 0\}$. 
Thus, the set $I_1( \g h^*_{2:d+1} / h_1^*)$ computed as in Lemma~\ref{lem:equivclassif} differs from $I_1(\g w^*)$ by at most $n=|I_0|$ entries and $I_1(\g w^*)$ must be one of the $2^n$ sets $I_1^{\g s} = \{i\in I : q_i = q_{i+N} = -1,\quad q_i =\sign ( \g h(\g w)^T \g z_i),\ \mbox{if}\ i\notin I_0,\ q_i = q_{i_k} = s_{k},\ \mbox{otherwise} \}$, where we indexed the entries in $I_0$ as $I_0 = \{i_1,\dots,i_n\}$ and $\g s\in \{-1,+1\}^n$ encodes the classification of the corresponding points. Solving the subproblem~\eqref{eq:regression} over the data points with index in $I_1^{\g s}$ for all $\g s\in\{-1,+1\}^n$ then yields at least one global solution in $\mathcal{W}_p^*$. 

Given that $J_p^*< \epsilon^p(N - d)$, we have $|I_1(\g w^*)| > d$ for all $\g w^*\in \mathcal{W}_p^*$, which implies $h^*_1 >0$ as follows. First, note that $I_1(\g w^*)\setminus I_0$ is not empty: by Lemma~\ref{lem:n}, $I_0$ contains no more than $d$ indexes within $I$ while $I_1(\g w^*)\subseteq I$ and $|I_1(\g w^*)| > d$. 
Then, by Lemma~\ref{lem:equivclassif} and the sign equalities~\eqref{eq:signeq}, for some $i\in I_1(\g w^*)\setminus I_0$, $\g z_i^T \g h^* < 0$ and $\g z_{i+N}^T \g h^* < 0$. Assume $h^*_1 = 0$, then $\g z_i^T \g h^* = -\g x_i^T\g h_{2:d+1}^* = - \g z_{i+N}^T\g h^*$, which shows a contradiction with the previous statement. Similarly, letting $h_1^*<0$ and using $\g z_i^T \g h^*<0$ yields $\g z_{i+N}^T \g h^* = -\g z_i^T \g h^* - 2\epsilon h_1^* > -2\epsilon h_1^* > 0$, and a contradiction with  $\g z_{i+N}^T \g h^* < 0$. 
Thus, $h^*_1 >0$.

{\em Part ii)} For $p=0$, given $h^*_1 >0$, $\sign(\g z_i^T\g h^*) = \sign(\g z_i^T\g h^* / h^*_1)$. Thus, if $\g s^*$ is a choice of $\g s$ yielding an optimal solution with $\g s^*_k = \sign(\g z_{i_k}^T\g h^*)$, $k=1,\dots,n$, by Lemma~\ref{lem:equivclassif}, $I_1(\g w^*) = I_1^{\g s^*}  = I_1( \g h^*_{2:d+1} / h_1^*)$. 
Given that for $p=0$ any $\g w$ such that $I_1(\g w) = I_1(\g w^*)$ yields the same cost $J_0(\g w) = J_0(\g w^*)$, we conclude that in this case $\g w = \g h^*_{2:d+1} / h_1^* \in \mathcal{W}_0^*$. Since $\g s^*_k\neq \sign(\g z_{i_k}^T\g h^*)$ only occurs for at most $n$ values of $k$, $I_1(\g w)$ deviates by at most $n$ entries from this case and $|I_1(\g w)| \leq |I_1(\g w^*)| + n$, yielding the second statement.
\end{proof}

Proposition~\ref{prop:gloabloptfromclassif} shows that a global minimizer of Problem~\ref{pb:min} can be obtained from a particular separating hyperplane of $\mathcal{Z}$. In addition, this hyperplane can be built from a subset of $\mathcal{Z}$ of cardinality $n\in[d,2d]$ as the one that passes through the origin of $\R^{d+1}$  and the $n$ points in the subset. Since any subset of $d$ points among these $n$ points yields the same hyperplane, it suffices to find one particular subset of $d$ points among $\mathcal{Z}$. Hence, the problem is reduced to a combinatorial search with $\binom{2N}{d}$ main iterations. 

This is formally stated in Algorithm~\ref{alg:regression} and the theorem below. 

\begin{theorem}\label{thm:beclassif}
Under Assumption~\ref{ass:generalposition} and given that subproblem~\eqref{eq:regression} can be solved in $\O(N^c)$ time with a constant $c\geq 1$ independent of $N$, Algorithm~\ref{alg:regression} solves Problem~\ref{pb:min} in $\O(N^{c+d})$ operations.
\end{theorem}
\begin{proof}
For all subsets of $d$ points of $\mathcal{Z}$, Algorithm~\ref{alg:regression} computes the hyperplane passing through these points and of orientation such that $h_1>0$. By Proposition~\ref{prop:gloabloptfromclassif}, at least one of these hyperplanes with normal $\g h$ is such that the inner loop over $\g s$ in Algorithm~\ref{alg:regression} finds a global solution.  

The computational complexity of Algorithm~\ref{alg:regression} is the number of subsets of $d$ points among $2N$ times the time needed for a single iteration, which includes computing a normal vector $\g h$, classifying $\mathcal{Z}$ with $\g h$, the inner loop over $\g s$ and the subproblem~\eqref{eq:regression}.  The normal $\g h$ of a hyperplane passing through the origin and $d$ points $\{\g z_{i_k}\}_{k=1}^d$ in $\R^{d+1}$ can be computed as a unit vector in the null space of  $[\g z_{i_1}, \dots, \g z_{i_d}]^T$, extracted in $\O(d^3)$. If $h_1 < 0$, then $-\g h$ is a normal for the hyperplane of opposite orientation. Given that $\g z_{i+N}^T \g h = -\g z_i^T \g h - 2\epsilon h_1$, the classification step takes $\O(N(d+2))$ instead of $\O(2N(d+1))$. The inner loop contains $2^n\leq 2^{2d}$ iterations with $\O(N^c)$ operations each to solve an instance of~\eqref{eq:regression} over at most $N$ data points. 
Overall, we obtain a time complexity of 
\begin{align*}
	T(N) &= \O\left(\begin{pmatrix}2N\\d\end{pmatrix} ( d^3 + (d+2)N + 2^{2d} N^c ) \right) \\
	&= \O\left( \frac{N^d}{d!}( d^3 + (d+2)N + 2^{2d} N^c ) \right) \\
	&= \O \left( N^{d+c}\right).
\end{align*}
\end{proof}

As an example, applying Theorem~\ref{thm:beclassif} for the saturated square loss, $\ell_{2,\epsilon}$, yields the exact Algorithm~\ref{alg:regression} for robust regression with complexity in the order of $\O(N^{d+1})$.

\begin{algorithm}
\caption{Exact $\ell_{p,\epsilon}$-regression \label{alg:regression}}
\begin{algorithmic}
\REQUIRE a data set $\{(\g x_i, y_i)\}_{i=1}^N$, a threshold $\epsilon>0$.
\STATE Initialize $J^* \leftarrow \epsilon^p N$. 
\FORALL {$S\subset \mathcal{Z}$ with cardinality $|S| = d$} 
	\STATE Compute the normal $\g h$ to the hyperplane passing through $S\cup\{\g 0\}$ with orientation such that $h_1\geq 0$.
	\IF {$h_1 \neq 0$ } 
		\STATE Classify the points:\\ $\forall i\in \{1,\dots,2N\},\ q_i = \sign_0(\g h^T\g z_i)$.
		\STATE Set $I_0 = \{i\in \{1,\dots,2N\} : q_i = 0\},\quad n=|I_0|$.
		\FORALL { $\g s \in\{-1,+1\}^n$ } 
			\STATE Set the entries of $\g q$ with index in $I_0$ to $\g s$
			\STATE Compute $I_1^{\g s} = \{i \leq N :  q_i = q_{i+N} = -1\}$
			\IF { $\epsilon^p(N-|I_1^{\g s}|) < J^*$ }
				\STATE Compute $\hat{\g w}$ as in~\eqref{eq:regression} with $I_1^{\g s}$ replacing $I_1(\g w^*)$.
				\IF{ $J_p(\hat{\g w}) < J^*$}
					\STATE Update $J^* \leftarrow J_p(\hat{\g w}),\ \g w^*\leftarrow \hat{\g w}$
				\ENDIF
			\ENDIF
		\ENDFOR
	\ENDIF
\ENDFOR
\RETURN $J^*, \g w^*$
\end{algorithmic}
\end{algorithm}

\paragraph{Remark} Note that a number of details can be included in Algorithm~\ref{alg:regression} to make it more efficient in practice. First, it can be easily parallelized. Then, a number of computations can be spared. For instance, since the changes in $\g s$ can only incur a difference of at most $n$ on $|I_1^{\g s}|$, the inner loop over $\g s$ can be skipped when $I_1$ computed with $q_i=\sign(\g h^T \g z_i)$ is such that $\epsilon^p(N-|I_1|) > J_p^* + \epsilon^p n$. 
Also, for $p=0$, since $J_p(\g w)=N-|I_1^{\g s}|$ is constant for a fixed classification, we do not need to solve~$\eqref{eq:regression}$ to update $J^*$ and $\g w^*$ can be computed only once at the end of the procedure.

\subsection{$\ell_{p,\epsilon}$-subspace estimation}
\label{sec:polysub}

The results for subspace estimation will be derived via a quadratic lifting of the classification problem of assigning point indexes to $I_1(\g B)$.
\begin{definition}[Veronese map of degree 2]
The Veronese map of degree 2 is the map
\begin{align*}
	\nu :\ &\R^d \rightarrow \R^D, \\
		 & \g x \mapsto [x_1^2, x_1x_2, \dots,\ x_2^2, x_2 x_3, \dots,\ x_{d-1}^2, x_{d-1} x_d,\ x_d^2]^T,
\end{align*}
where $D=d(d+1)/2$.
\end{definition}
Using this map, we build a new data set
\begin{equation}\label{eq:Zsubspace}
	\mathcal{Z}=\{\g z_i\}_{i=1}^{N}, \quad\mbox{with } \g z_i =\begin{bmatrix}-\epsilon^2\\\nu(\g x_i)\end{bmatrix}
\end{equation}
and establish a correspondence between subspace estimation and the classification of this data set.
\begin{lemma}\label{lem:equivclassifsubspace}
Given a subspace basis $\g B=[\g b_1,\dots,\g b_{d_s}]\in \R^{d\times d_S}$ such that $\g B^T\g B=\g I$ and a set $\mathcal{Z}$ as in~\eqref{eq:Zsubspace}, the set $I_1(\g B)$ defined in~\eqref{eq:I1subspace} is given by
\begin{equation}\label{eq:I1subspaceclassif}
	I_1(\g B) = \{i\in I : \g h(\g B)^T \g z_i < 0 \},
\end{equation}
where $h_1(\g B) = 1$ and $\g h_{2:D+1}(\g B) = \g s - \sum_{j=1}^{d_s} \nu(\g b_j)$ with the selection vector\footnote{$\g s$ is defined by $s_{l} = 1$ iff $l\in\{l_k\}_{k=1}^D$ with $l_1 = 1$ and $l_k = l_{k-1} + d - k + 2$.} $\g s\in\{0,1\}^D$ such that $\forall \g x\in\R^d$, $\nu(\g x)^T\g s = \sum_{k=1}^d x_d^2$. 
\end{lemma}
\begin{proof}
For any $i\in I$, we have
\begin{align*}
 \|(\g I - \g B\g B^T)\g x_i\| < \epsilon &\Leftrightarrow \g x_i^T(I - \g B\g B^T)^T(I - \g B\g B^T)\g x_i < \epsilon^2 \\
 &\Leftrightarrow  \g x_i^T(I - \g B\g B^T)\g x_i < \epsilon^2 \\
 &\Leftrightarrow  \g x_i^T\g x_i - \sum_{j=1}^{d_s} \g x_i^T(\g b_j\g b_j^T)\g x_i < \epsilon^2 \\
 &\Leftrightarrow  \nu(\g x_i)^T\g s - \sum_{j=1}^{d_s} \nu(\g x_i)^T\nu(\g b_j) < \epsilon^2 \\
 &\Leftrightarrow  \nu(\g x_i)^T\left[\g s - \sum_{j=1}^{d_s} \nu(\g b_j)\right] - \epsilon^2 < 0 \\
 &\Leftrightarrow  \g h(\g B)^T \g z_i < 0 .
\end{align*}
Thus, the definition of $I_1(\g B)$ coincides with~\eqref{eq:I1subspaceclassif}.
\end{proof}

We will also need the following manipulation of Assumption~\ref{ass:generalpossubspace}.
\begin{lemma}\label{lem:genposS}
Under Assumption~\ref{ass:generalpossubspace}, no hyperplane of $\R^{D+1}$ passes through the origin and more than $D$ points of $\mathcal{Z}$ as defined in~\eqref{eq:Zsubspace}. 
\end{lemma}
\begin{proof}
Assumption~\ref{ass:generalpossubspace} and the fact that the Veronese map is biregular, i.e., the image of points in general position in $\R^d$ under the Veronese map $\nu$ are again in general position, imply that $\mathcal{V}=\{\nu(\g x_i)\}_{i=1}^N$ is in general position in $\R^D$. Therefore, there is no hyperplane of $\R^D$ that passes through more than $D$ points of this set. Since the projection onto $\R^D$ of any hyperplane of $\R^{D+1}$ passing through the origin and more than $D$ points of $\mathcal{Z}$  must pass through more than $D$ points of $\mathcal{V}$, there is no such hyperplane. 
\end{proof}

\begin{proposition}\label{prop:classifsubspace}
Under Assumption~\ref{ass:generalpossubspace}, in $\R^{D+1}$, there is a hyperplane of normal $\g h\in \R^{D+1}$ passing through the origin and exactly $D$ points of $\mathcal{Z}$ as in~\eqref{eq:Zsubspace} such that a global solution to Problem~\ref{pb:subspace} can be found by solving $2^{D}$ subspace estimation subproblems~\eqref{eq:subspaceest}. 
\end{proposition}
\begin{proof}[Proof sketch]
The proof works as the first part of the one of Proposition~\ref{prop:gloabloptfromclassif}, except that we use Lemma~\ref{lem:equivclassifsubspace} instead of Lemma~\ref{lem:equivclassif} and Lemma~\ref{lem:genposS} instead of Lemma~\ref{lem:n} to bound from above the number of points lying on the hyperplane by $D$ instead of $2d$. It is also simpler as the classification of $\g x_i$ is directly given by the one of $\g z_i$ without taking into account an additional $\g z_{i+N}$.
\end{proof}
Compared with the regression case and Proposition~\ref{prop:gloabloptfromclassif}, Proposition~\ref{prop:classifsubspace} does not ensure $h_1>0$ and thus the algorithm needs to test both orientations for every hyperplane. This yields Algorithm~\ref{alg:subspace} for which we have the following result.
 
\begin{theorem}\label{thm:beclassifsubspace}
Under Assumption~\ref{ass:generalpossubspace} and given that subproblem~\eqref{eq:subspaceest} can be solved in $\O(N^c)$ time for a constant $c\geq 1$ independent of $N$, Algorithm~\ref{alg:subspace} solves Problem~\ref{pb:subspace} in $\O(N^{c+d(d+1)/2} )$ operations.
\end{theorem}
\begin{proof}
For all subsets of $D$ points of $\mathcal{Z}$, Algorithm~\ref{alg:subspace} computes the hyperplane passing through the points. By  Proposition~\ref{prop:classifsubspace}, at least one of these hyperplanes with normal $\g h$ is such that a global minimizer $\hat{\g B}$ can be recovered by solving an instance of~\eqref{eq:subspaceest} in the inner loops over $S$ and $s$.

The computational complexity, $T(N)$, of Algorithm~\ref{alg:subspace} is the number of subsets of $D$ points among $N$, $\binom{N}{D}$, times the time needed for computing a normal vector $\g h$, $\O(D^3)$, classifying $\mathcal{Z}$ with $\g h$, $\O(DN)$, and performing the inner loops over $S$ and $s$ with the subproblem~\eqref{eq:subspaceest}, $\O(2\times 2^D N^c)$:
$$
	T(N) = \O\left(\binom{N}{D} (D^3 + DN + 2^{D+1} N^c)\right)
	= \O ( N^{D+c}) 
$$
Recalling that $D=d(d+1)/2$ completes the proof.
\end{proof}

As an example, applying Theorem~\ref{thm:beclassifsubspace} for the saturated square loss, $\ell_{2,\epsilon}$, yields the exact Algorithm~\ref{alg:subspace} for robust subspace estimation with complexity in the order of $\O(N^{1+d(d+1)/2})$.

\begin{algorithm}
\caption{Exact $\ell_{p,\epsilon}$-subspace estimation \label{alg:subspace}}
\begin{algorithmic}
\REQUIRE a data set $\{\g x_i\}_{i=1}^N$, a threshold $\epsilon>0$.
\STATE Initialize $J^* \leftarrow \epsilon^2 N$.
\FORALL {$S_h\subset \mathcal{Z}$ with $|S_h| = D$} 
	\STATE Compute the normal $\g h$ of the hyperplane passing through the points in $S_h\cup\{\g 0\}$.
	\STATE Classify the points: $\forall i\in I$, $q_i =  \sign_0(\g h^T\g z_i)$.
	\FORALL {$S\subseteq S_h$} 
		\FORALL[for all orientation] {$s\in\{-1,+1\}$}
			\STATE $I_1 = \{i \leq N :  q_i = s\} \cup \{i\leq N : \g x_i\in S\}$ 
			\STATE Compute $\hat{\g B}$ as in~\eqref{eq:subspaceest}  with $I_1$ replacing $I_1(\g B^*)$. 
			\IF { $J_p^S(\hat{\g B}) < J^*$}
				\STATE Update $J^* \leftarrow J_p^S(\hat{\g B})$, $\g B^*\leftarrow \hat{\g B}$.
			\ENDIF
		\ENDFOR
	\ENDFOR
\ENDFOR
\RETURN $J^*, \g B^*$
\end{algorithmic}
\end{algorithm}

\section{Random sampling}
\label{sec:random} 

We now detail more practical (but approximate) variants of the algorithms above. These are based on random sampling of subsets of points rather than a complete enumeration. As such, they share some features with the well-known RANSAC method \citep{Fischler81} for robust estimation. 

\paragraph{RANSAC}
The RANSAC method iterates through small subsets of $s$ points and estimates a linear model at each iteration from these $s$ points only. Then, the model that best approximates the maximum number of points can be retained, the rationale being that it should be possible to find a small subset of points within the set inliers and thus to estimate a good model from inliers only.  
However, this approach has two major drawbacks. First, since only $s$ points are used to estimate the models, even if all the $\binom{N}{s}$ subsets are completely enumerated, the RANSAC cannot guarantee the recovery of an optimal solution, unless $s$ equals the {\em unknown} number of inliers. And in this case, the computational complexity becomes exponential in $N$ as soon as we assume that a certain fraction of the $N$ points are inliers. The second weakness is related to the tuning of $s$ which should be made in accordance with the noise level and the fraction of inliers: larger values of $s$ tend to filter more efficiently the noise but also generate more subsets corrupted by outliers and decrease the probability of selecting a subset of inliers only. 

\paragraph{Random sampling variants of Algorithms~\ref{alg:regression}--\ref{alg:subspace} } 
Inspired by the RANSAC method, we can develop approximate variants of Algorithms~\ref{alg:regression}--\ref{alg:subspace}, in which the complete enumeration of the subsets $S$ is replaced by random sampling. 
Note that though this approach also relies on randomly sampling subsets of points, it remains quite different from the RANSAC. Indeed, the RANSAC directly estimates a linear model from the points in a subset, whereas here the subsets of points are only used to determine the classification of the entire data set into inliers and outliers. Therefore, it remains possible for such approximate variants to find the optimal set of inliers from which it can compute an exact solution.
As in the RANSAC, the accuracy of the resulting model mostly depends on the number $N_{iters}$ of iterations and thus of tested subsets. However, contrarily to the RANSAC, there is no other hyperparameter to tune: the size of the subsets is fixed to $d$ (for regression) or $D$ (for subspace estimation) from the analysis of Sect.~\ref{sec:exact}.

\section{Experiments}
\label{sec:exp} 

This section reports a few numerical results regarding first the exact algorithm for regression and then its approximate variant based on random sampling. 

\subsection{Exact algorithm}

We here evaluate the gain in computing time offered by the exact polynomial algorithms. In particular, we focus on Algorithm~\ref{alg:regression} for $\ell_{0,\epsilon}$-regression and compare its computing time with that needed to solve Problem~\ref{pb:min} with standard tools that can also compute exact solutions such as mixed-integer programming solvers. Specifically, for $p=0$, Problem~\ref{pb:min} can be reformulated as the mixed-integer linear program (MILP)
\begin{align}\label{eq:milp}
\min_{\g w\in [-W,W]^d, \g\beta\in\{0,1\}^N} &  \sum_{i=1}^N \beta_i\\
	s.t.\ &  y_i- \g w^T \g x_i - \epsilon \leq M\beta_i,\quad i=1,\dots,N, \nonumber\\
	   & \g w^T \g x_i - y_i - \epsilon\leq M\beta_i,\quad i=1,\dots,N, \nonumber
\end{align}
with binary variables $\beta_i$ encoding $\I{|y_i-\g w^T\g x_i| < \epsilon}$ and a constant $M$ large enough to upper bound any absolute error term, i.e., set as $M\geq \max_{i\in\{1,\dots,N\}} |y_i| +  d W \|\g x_i\|_\infty$. 

Experiments are conducted for random data sets of increasing size $N$ generated by $y_i=\g x_i^T \g w_0 + \xi_i + \nu_i$ for $\g x_i$ uniformly distributed in $[-5,5]^d$, $\g w_0=[1, -0.5, 0.8]^T$, a zero-mean Gaussian noise $\xi_i$ of variance $0.1$ and an outlying Gaussian noise $\nu_i$ of mean $100$ and variance $1000$ added only to $40\%$ of the data. Computing times reported in Table~\ref{tab:milp} refer to a parallel implementation in Matlab of Algorithm~\ref{alg:regression} and the use of CPLEX for solving~\eqref{eq:milp} (which also benefits from parallel processing) on a laptop equipped with an i5-7440HQ processor at 2.8GHz. 

\begin{table}
\centering
\caption{Comparison of the computing time  (mean $\pm$ standard deviation over 4 trials) of Algorithm~\ref{alg:regression} and that of solving the MILP~\eqref{eq:milp}. ``n/a" appears when the method did not terminate after $10$ hours.\label{tab:milp}}
\begin{tabular}{l|c|c|c}\hline
$N$ 	& 100 & 300 & 1000 \\\hline
MILP~\eqref{eq:milp}	& $0.5\pm 0.2$ s & $11.3 \pm 18.5$ min	& n/a \\ 
Algorithm~\ref{alg:regression}	& $3.9\pm 0.1$ s & $43.2\pm 0.8$ s & $25 \pm 0.4$ min \\\hline
\end{tabular}
\end{table}

While CPLEX can compete with Algorithm~\ref{alg:regression} on small data sets ($N=100$), the worst-case exponential complexity of MILPs makes it far slower when $N$ increases. In addition, its computing time highly varies between different trials for the same data set size, whereas that of Algorithm~\ref{alg:regression} is not influenced by the data and remains pretictable for given problem dimensions.

\subsection{Approximate variant}

We now compare the variant of Algorithm~\ref{alg:regression} for $p=2$ proposed in Sect.~\ref{sec:random} with the RANSAC when both algorithms use the same number of iterations ($N_{iters}=3000$). Figure~\ref{fig:ransac} shows how the errors, $\|\g w_0 - \g w\|_2/\|\g w_0\|_2$, between the target vector $\g w_0$ and the estimate $\g w$ evolve with the fraction $r$ of outliers ($\nu_i\neq 0$ for $rN$ data points). More precisely, we plot the mean errors over 100 trials with $\g w_0$ uniformly drawn in $[-5,5]^d$, $d=4$. 
The RANSAC uses $s=2d$, but similar results are obtained for other values.
The plot in Fig.~\ref{fig:ransac} indicates that the proposed random sampling can perform better than the RANSAC in highly perturbed regimes with $70\%$ of outliers or more. 
In addition, the random sampling variant of Algorithm~\ref{alg:regression} is obviously much faster than its exact version, leading to computing times similar to those of RANSAC and about 0.1 second in these experiments.

\begin{figure}
\centering
\includegraphics[width=0.5\linewidth]{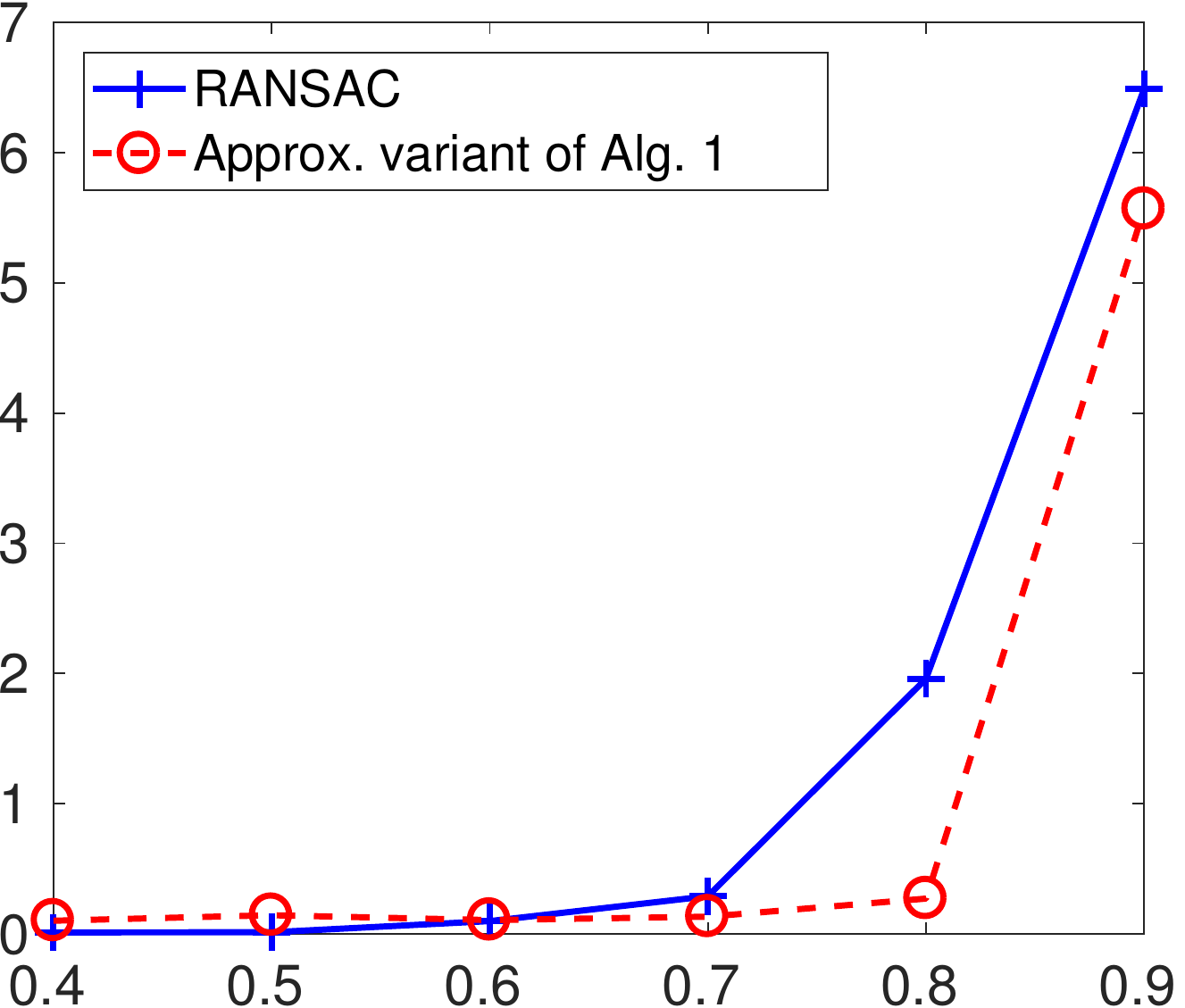}
\caption{Error versus the fraction of outliers.\label{fig:ransac}}
\end{figure}

\section{Conclusions}
\label{sec:conclusion} 

The paper analyzed the complexity of globally minimizing saturated loss functions for robust regression and subspace estimation with respect to the number of data. 
By deriving explicit connections between these estimation problems and linear classification, the paper could build on recent results on the enumeration of linear classifications to show that these global optimization problems have no more than a polynomial complexity in the number of data. Experiments showed that this provides a significant gain in speed when compared to a mixed-integer programming approach.

However, the developed algorithms have an exponential complexity wrt. the data dimension, which strongly limits their practical use. Therefore, approximate variants were proposed based on random sampling. Experiments showed that these variants can yield an increase of accuracy for a similar computational cost when compared with another classical method based on random sampling, namely, the RANSAC.

\end{document}